\documentclass{article}

\usepackage{microtype}
\usepackage{graphicx}
\usepackage{subfigure}
\usepackage{booktabs} 

\usepackage{hyperref}

\usepackage[accepted]{icml2019}

\icmltitlerunning{Quantifying Perceptual Distortion of Adversarial Examples}

\usepackage{amsmath}
\usepackage{amsthm}
\usepackage{amsfonts}
\usepackage{mathtools}
\newcommand{\eps}{\epsilon}
\newtheorem{theorem}{Theorem}
\newtheorem{definition}{Definition}

\usepackage{appendix}
\usepackage{etoolbox}

\BeforeBeginEnvironment{appendices}{\clearpage}

\usepackage{subfigure}
\usepackage{comment}
\newcommand{\norm}[1]{\left\lVert#1\right\rVert}
\newcommand{\abs}[1]{\left\lvert#1\right\rvert}
\newcommand{\paren}[1]{\left(#1\right)}

\newcommand{\setmath}[1]{\left\{#1\right\}}

\newcommand{\linf}{\ell_{\infty}}
\newcommand{\suchthat}{~\middle|~}

\DeclareMathOperator*{\argmax}{arg\,max}
\DeclareMathOperator*{\argmin}{arg\,min}

\usepackage{amsmath}
\usepackage{amsfonts}
\usepackage{mathtools}
\newtheorem{lemma}{Lemma}
\usepackage{subfigure}
\usepackage{comment}

\allowdisplaybreaks

\begin{document}

\twocolumn[
\icmltitle{Quantifying Perceptual Distortion of Adversarial Examples}



\icmlsetsymbol{equal}{*}

\begin{icmlauthorlist}

\icmlauthor{Matt Jordan}{utcs}
\icmlauthor{Naren Manoj}{utcs}
\icmlauthor{Surbhi Goel}{utcs}
\icmlauthor{Alexandros G. Dimakis}{ece}

\end{icmlauthorlist}

\icmlaffiliation{utcs}{Department of Computer Science, University of Texas at Austin}
\icmlaffiliation{ece}{Department of Electrical and Computer Engineering, University of Texas at Austin}

\icmlcorrespondingauthor{Matt Jordan}{mjordan@cs.utexas.edu}

\icmlkeywords{Machine Learning, ICML}

\vskip 0.3in
]



\printAffiliationsAndNotice{}  

\begin{abstract}

Recent work has shown that additive threat models, which only permit the addition of bounded noise to the pixels of an image, are insufficient for fully capturing the space of imperceivable adversarial examples. For example, small rotations and spatial transformations can fool classifiers, remain imperceivable to humans, but have large additive distance from the original images. In this work, we leverage quantitative perceptual metrics like LPIPS and SSIM to define a novel threat model for adversarial attacks.

To demonstrate the value of quantifying the perceptual distortion of adversarial examples, we present and employ a unifying framework fusing different attack styles. We first prove that our framework results in images that are unattainable by attack styles in isolation. We then perform adversarial training using attacks generated by our framework to demonstrate that networks are only robust to classes of adversarial perturbations they have been trained against, and combination attacks are stronger than any of their individual components. Finally, we experimentally demonstrate that our combined attacks \textit{retain the same perceptual distortion but induce far higher misclassification rates} when compared against individual attacks. 
\end{abstract}

\section{Introduction}
Modern deep neural networks are demonstrating groundbreaking performance for image classification and other related tasks but were recently discovered to be vulnerable to input manipulations.  
Adversarial examples are slight alterations to an input that completely change the output of a classifier but are imperceivable by humans. 
Adversarial examples have received significant and justified recent attention, see e.g. \cite{Szegedy2013,Goodfellow2014-vh,Madry2017-ia, Carlini2017-qm, Su2017-op, xiao2018spatially, Engstrom2017-ib}. It was recently demonstrated that even real-world physical objects can be manufactured to fool classifiers \cite{Athalye2017-aa, Kurakin2016-qg} under various lighting and sensing conditions. Beyond the obvious security concerns, adversarial examples illustrate that modern complex models can make correct predictions for completely wrong reasons.

One fundamental issue for all security problems is \textit{precisely defining a threat model}. Obviously, it is easy to fool any classifier by replacing the input with a vastly different image, but adversarial examples have to remain visually imperceivable to humans. This notion of perceptual similarity is loosely defined and has been classically hard to quantify, though there is a rich literature in constructing perceptual similarity metrics \cite{Donahue2016-ip, Wang2004-ky, Zhang2018-mw}. However, 
$\ell_\infty$ bounds are most commonly used as a proxy for a perceptual metric in the domain of adversarial examples; every pixel is allowed to be modified independently by only a small amount. While it is true that adding a small amount, in some $\ell_p$ sense, of noise to an image will generate images that are perceptually indistinguishable, the converse is not necessarily true: not all perceptually similar perturbations can be generated by addition of small $\ell_p$ noise. Specifically, it has been shown that tiny rotations~\cite{Engstrom2017-ib} or spatial transformations~\cite{xiao2018spatially} suffice to fool an image classifier. 
 
We motivate our work as follows: if human perception is robust to the addition of small $\ell_\infty$ noise as well as small rotations, translations and spatial transformations, \emph{then it is likely that human perception is robust to combinations of these types of transformations}. Referring to $\ell_p$-bounded noise as attacks that impact the texture of an image and spatial transformations as attacks that impact the geometry of an image, we generate attacks that leverage the adversarial power of modifying the both the texture and geometry of an image by composing various attacks that each perform one type of perturbation. 

In this paper we argue that the most relevant threat model is the one governed by human perception. We phrase adversarial examples as a search over a parameter space of perturbations and demonstrate that including parameters for both texture and geometry allows a more effective search over the space of perceptually similar images. We generate adversarial attacks using both texture and geometry transformations that induce significantly higher misclassification rates than either technique alone, while retaining the same or more perceptual similarity to the original images. In the defense domain, we show that networks adversarially trained against either technique alone are vulnerable to the other technique, lending empirical credibility to the claim that there exist orthogonal attack styles \cite{Engstrom2017-ib}, and that state-of-the-art adversarially defended networks overfit on $\ell_p$-style attacks \cite{overfitt}.

\subsection{Our Contributions}
\begin{itemize}

    \item  We demonstrate that our proposed combination attacks are stronger compared to additive and spatial transformation attacks \textit{of equal or lesser perceptual distortion}. 

    \item 
    For our extensive experimental validation, we developed a new PyTorch toolbox for creating combinations of adversarial attacks using percetual metrics, geometrical transformations and various other types of image transformations. Our package, called \href{https://github.com/revbucket/mister_ed}{\texttt{mister\_ed}} , consists of more than $5$ thousand lines of code and streamlines the design of perceptual adversarial attacks and defenses. We expect that our toolbox will facilitate research in perception-based adversarial machine learning.

    \item We prove that an attack that combines an addition of $\ell_\infty$-constrained noise and a constrained spatial transformation can generate a larger class of images than either attack style alone under some mild assumptions on the image that hold almost always for real images. 
    
   \item We present evidence to suggest that various attack styles are orthogonal: networks defended against a particular style of attack are not robust to attack styles that were not defended against. 
\end{itemize}

\subsection{Related Work}

Adversarial examples are modifications to the input of a machine learning algorithm that are carefully chosen to drastically change the algorithm output. The first instances of these arose in 2004, where carefully-crafted emails were designed to evade spam filtering algorithms \cite{dalvi2004adversarial,lowd2005adversarial}. Adversarial machine learning algorithms have become an extremely popular area of study in recent years; see also the comprehensive survey by Biggio et. al \cite{biggio2017wild}. In the image classification domain, most of the literature has focused on a threat model defined by using only $\ell_p$ perturbations: attackers are only allowed to add some $\ell_p$-bounded noise to an image \cite{carlini2017adversarial, Madry2017-ia}, where the amount of noise is typically constrained to not be too large by a threat model. This noise is optimized to maximize the classifier loss by leveraging gradients of the neural network. Recently there has been some work in constructing adversarial examples by rotating, translating or performing spatial transformations on images \cite{Engstrom2017-ib,xiao2018spatially}. These are distinct threat models, as even a minor translation or rotation of an image could generate a large $\ell_p$ distance. In general, any transformation can be efficiently leveraged in an adversarial attack as long as the transformation is differentiable.

While it is clear that adding perceptual noise $\delta$ to image $x$ can be viewed as $x + \delta$, it is less clear how arbitrary spatial transformations can be modeled while maintaining differentiability. Spatial transformer networks \cite{Jaderberg2015-xm} were developed as a plug-and-play module to use as a layer in a DNN, similar in spirit to a convolutional layer. Spatial transformation layers work by computing a grid of new pixel locations and using bilinear interpolation to compute the pixel values in the new image. The grid is computed as a function of an image as well as learned parameters. Maximally these parameters can define the exact location of each new pixel, but the class of affine transformations could be modeled using only 6 parameters (4 for rotation, 2 for translation). Critically, bilinear interpolation is differentiable with respect to the parameters so we can leverage first-order methods to efficiently conduct adversarial attacks.

Although computing a quantifiable metric for the perceptual distance between two images is inherently challenging due to the subjective nature of human perception, there is a rich literature for measuring the similarity between two images. $\ell_p$ norms serve as one such proxy, though these are not robust to minor spatial transformations such as translation or rotation and have been shown to not be sufficient in adequately describing human perception \cite{Sharif2018-fd}. Other efficiently-computable similarity scores exist, such as PSNR and SSIM \cite{Wang2004-ky}. Recently work \cite{Zhang2018-mw} has demonstrated that comparing normalized cosine distances in activation layers of a DNN trained to classify real images can perform exceptionally well as a similarity score, as backed up by an extensive human study. Quantifying perceptual distortion is of high interest to computer vision, and it is safe to assume that techniques for computing image similarity will only improve over time. It is for this reason that our approach towards adversarial attacks attempts to be metric-agnostic, with the only requirements being that we require the similarity scores to be symmetric, reflexive and efficiently computable.

\section{Preliminaries} 
Throughout this paper, we let $f$ be a multiclass-classifier that maps real vectors with $n$ features into one of $C$ classes. We suppose that $f$ is trained on some distribution of real data $X$, and we refer to elements of the input space $\mathbb{R}^n$ as images. The power set operator is defined by $\mathcal{P}(\cdot)$. We are only interested in adversarial attacks where the adversary has strict constraints on how it can perturb the given inputs. Certainly this is necessary, as discussing the strength of an attack, or robustness to attacks is crucially dependent upon what the adversary is allowed to do: without such constraints, an adversary is allowed to replace a given image with any image. To this end, we define a threat model as follows
\begin{definition}[Threat Model]
A $\textbf{threat model}$ is a function $t : \mathcal{P}(\mathbb{R}^n) \rightarrow \mathcal{P}(\mathbb{R}^n)$ that takes as input a set of images and outputs a set of images that are valid perturbations the adversary can perform to an input image.\footnote{Note that, traditionally, the input to a threat model is the singleton image for which we are generating adversarial examples and the output is the set of valid adversarial examples. We define it more generally over sets as we will compose various attacks, though we overload notation as such: $t(x) := t(\{x\})$.}
\end{definition}

As an example, typical threat models considered for MNIST allow $\ell_\infty$ perturbations of $0.3$, and for CIFAR-10, $\ell_\infty$ perturbations of $8/255$ are typically allowed. With a well-defined threat model, we define an adversarial attack as follows:

\begin{definition}[Adversarial Attack]
An \textbf{adversarial attack} is a function $A: \mathbb{R}^n \rightarrow \mathbb{R}^n$ parameterized by both a classifier and a threat model mapping an image to an adversarial perturbation of the image such that for all $x\in \text{\emph{supp}}(X),\; A(x) \in t(x)$. We say the attack is \textbf{successful} on image $x$ if $f(x) \neq f(A(x))$.
\end{definition}

Our work applies to the \emph{white box} setting, where the adversary has access to both the classifier $f$ and its gradients. We focus on untargeted adversarial attacks, though our work can be easily extended to the targeted setting. For a given threat model, we now phrase the problem of defining the function $A$ as an optimization problem where 
$$A(x) := \argmax_{y\in t(x)} Loss(x, y)$$
for some loss that is maximized when $f(x) \neq f(y)$. This is typically either solved using first order techniques where projection back onto the feasible set $t(x)$ is performed after each gradient step, or, in the case of the attacks presented in \cite{Carlini2017-qm}, the threat model is brought to the objective through a Lagrange term and a binary search over the scaling factor is performed. It should be noted that in the Carlini-Wagner attacks, the threat model is typically not strict and the optimization aims to minimize some distance function while still inducing misclassification; to incorporate this under our strict threat model framework, we define a threat model and reject any generated examples outside of the threat model.

\section{Our Approach}
In this section, we detail our primary contributions. We first reiterate our motivation through a toy example. We next formalize our attack procedure and provide examples of attacks we compose. We then show that our motivation for attack composition is grounded in theory. 

\subsection{Motivating Example}
\begin{figure*}
\centering
    \subfigure{\includegraphics[width=0.7\columnwidth]{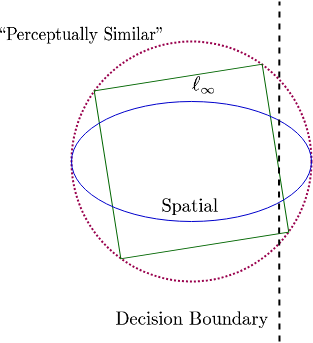}}
    \hfill
    \subfigure{\includegraphics[width=\columnwidth]{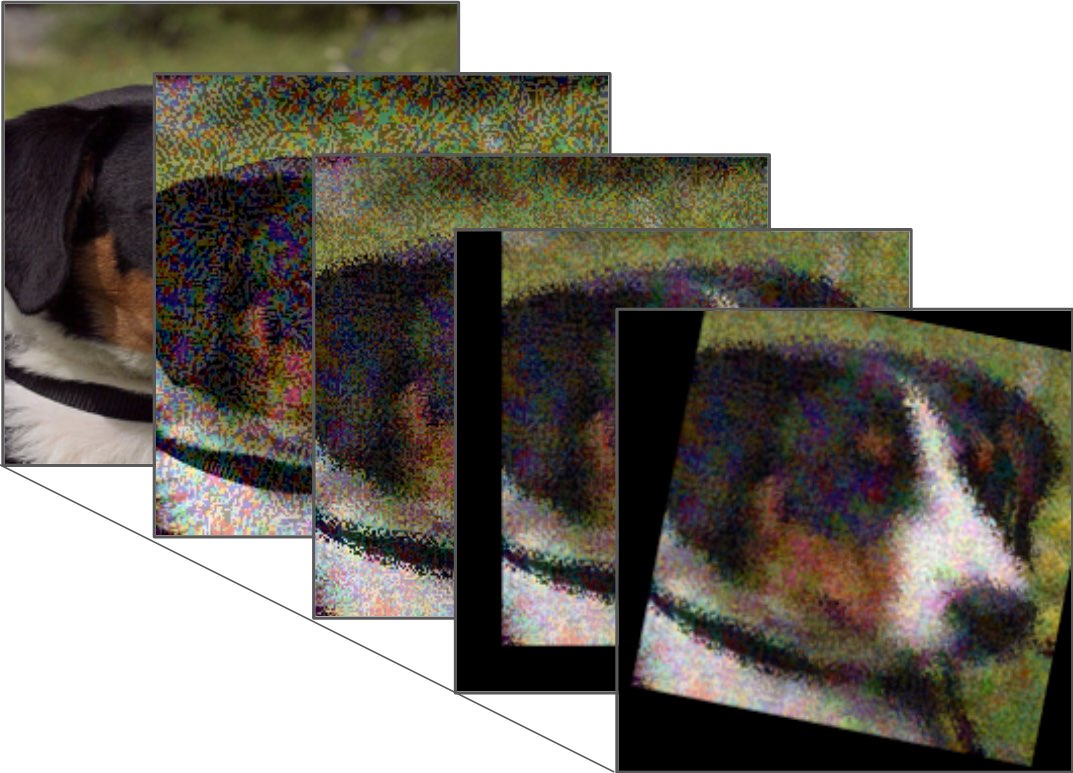}}
    \caption{Left: Observe that the space of all perceptually small perturbations is not necessarily covered by the space of small $\ell_p$ or flow perturbations. Right: We sequentially perturb this example image from ImageNet using larger-than-necessary perturbations for expository purposes. The perturbations are, from left to right: original, $\delta$-addition, flow perturbation, translation, rotation.}
    \label{fig:perceptual}
\end{figure*}

Consider $\eps$ balls centered at the origin induced by various metrics in some space, such as those shown in Figure \ref{fig:perceptual}. These balls may have some nonempty intersection; however, there are regions in our space that are within $\eps$ of the origin in some metric but not in the others. Now, if we let these metrics be the ``magnitudes'' of various perturbations, the value of this line of reasoning becomes clear. Specifically, suppose we perform some spatial transformation on our image, such as a translation. Although a small translation is usually nearly imperceptible, the $\ell_{p}$ distance between the new image and the original image is large. Thus it is likely that there exist some perceptually similar images which cannot be attained by minimizing some $\ell_p$ norm between the generated example and the original image. We can extend this same line of reasoning to other measures of perturbation magnitude. 

Recall that our goal is to generate images that are perceptually similar but result in differing model evaluations. Phrased differently, we seek to efficiently parameterize the ``perceptual ball'' around a given example $x$ in such a way that permits efficient exploration. Certainly this perceptual ball is covered by an $\ell_p$ ball of sufficiently large size, though even a small rotation has a potentially large $\ell_p$ norm while having a minimal perceptual difference from the original example. Thus, should our parameterization of our adversarial attack depend only on an $\ell_p$-constrained additive $\delta$, it is unlikely that standard iterative techniques will generate an adversarial example corresponding to a small spatial transformation. This suggests that a reparameterization of our search space is necessary. 

We can better approximate the perceptual ball by the composition of several small, parameterized perturbations to an image. Each perturbation is of a different flavor - for instance, we could add an $\ell_\infty$-constrained $\delta$ to our image and then perform an $\epsilon$ translation. By composing several attack classes together, our hope is to approximate this region of perceptually similar images and parameterize it in such a way that standard iterative optimization techniques will be able to explore the entire set of perceptually similar images. 

As an alternative way of viewing the notion above, consider the problem of image classification: given a distribution over images, typically viewed as vectors in $\mathbb{R}^n$, and their corresponding labels, a classifier attempts to learn the subset of $\mathbb{R}^n$ corresponding to each label. That is, it attempts to efficiently characterize a potentially nonconvex space. Neural nets, which simply repeatedly compose linear operators and nonlinearities, perform tremendously well on this domain. Borrowing this intuition, we can better approximate the perceptual ball by sequentially applying various classes of attacks.

\subsection{Attack Procedure}
\begin{figure*}
    \centering
    \includegraphics[width=\textwidth]{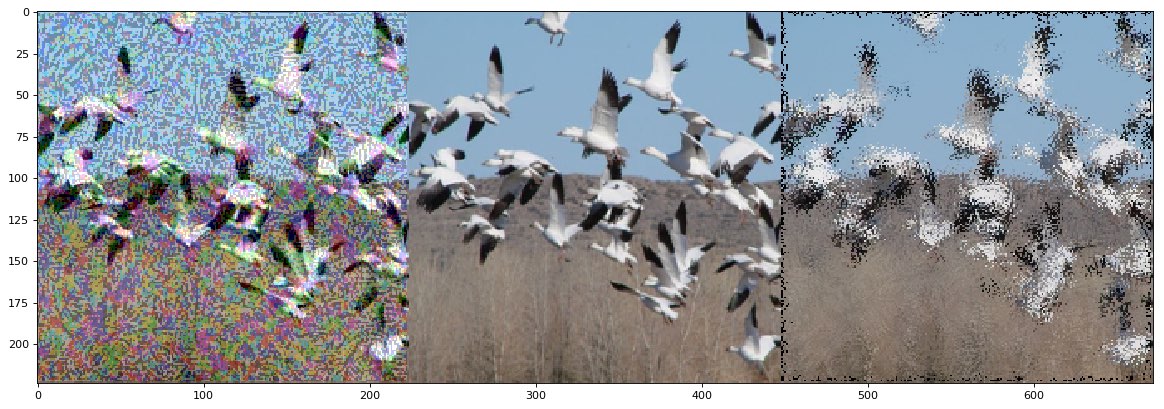}
    \caption{Here are adversarial examples generated from an image in ImageNet. The center image is the unperturbed image. The left image is generated with a large $\delta$-addition. The right image is generated with a large non-smooth spatial perturbation. Note that the perturbation generated by each loss is far larger than is necessary to result in a misclassification.}
    \label{fig:adversarial_birds}
\end{figure*}
Building on the intuition described above, we generate attacks as a composition of several smaller attacks, each of which is called a ``layer". Each layer is parameterized by its own threat model. The total threat model of the attack is simply the composition of the constituent threat models.
\[
t = t_L \circ t_{L-1} \circ \ldots \circ t_1.
\]
where $t_i$ is the threat model of layer $i$ for $i \in \{1, \ldots, L\}$.

For each image that is adversarially attacked, each layer maintains a set of parameters that are optimized during the attack generation process. Typically the threat model for each layer limits the set of parameters that are valid, e.g. each layer's parameters are $\ell_p$-bounded. Similar to DNNs, there are multiple classes of layers that can be applied. Here we list a few that we use, noting that this procedure could work with any composition of differentiable transformations.

\begin{itemize} 
    \item \textbf{Delta additions} are parameterized by $\delta$, where the transformation applied to image $x$ yields $x + \delta$. We optimize over the choice of $\delta$, where the threat model usually constrains $\delta$ to be bounded in an $\ell_p$ sense. The number of parameters for this layer is exactly equivalent to the number of pixels times the number of channels in $x$. This class of transformation is what is considered for all $\ell_p$ style attacks. 
    \item \textbf{Affine transformations} are parameterized by an affine transformation matrix $\theta$, which is a $2\times 3$ matrix. We constrain these to be rotations, translations, and dilations. Threat models for layers of this class typically constrain the maximum allowable angle, shift, and scale for rotations, translations and dilations respectively. Rotations and translation layers were considered in  \cite{Engstrom2017-ib}.
    \item \textbf{Flow networks} are parameterized by a full sampling grid. Each output pixel maintains a horizontal and vertical coordinate dictating where it is sampled from the input image~\cite{Jaderberg2015-xm}. While affine transformations can also be parameterized by a sampling grid, we use the term ``flow networks'' to specifically describe transformations that are maximally parameterized: the number of parameters is twice the total number of pixels in the input image\footnote{Note that we classify affine transformations and flows differently since their threat models vary. For example, a rotation of $5^\circ$ may be within our threat model whereas the flow change caused by this rotation on the boundaries could possibly violate the flow threat model.}. Threat models in this class constrain the $\ell_\infty$ distance between the coordinates of the original and sampled pixel. This class of transformations is considered by the work in \cite{xiao2018spatially}.
\end{itemize}
Broadly, we refer to affine transformations and flow networks as ``spatial transformations.''

\subsection{Threat Model Search Space}

Our reasoning for composing different attack types is intuitive: there exist certain images that have regions that a spatial attack can exploit well but an additive attack cannot and vice versa. For instance, flow attacks are powerful in image regions with high contrast, such as borders or object boundaries, but are weak in regions with low contrast, such as flat backgrounds or image interiors.
On the other hand, additive attacks are not as powerful as flow in high contrast areas as they can make only small perturbations whereas flows can take the value of its neighbouring pixels. By expanding our threat model, we can obtain the best of both worlds, thereby better parameterizing the space of perceptually similar images. We emphasize that by constraining the amount of additive noise and spatial transformations allowed we are able to better approximate a perceptual ball \emph{in a principled fashion}. We can also demonstrate that the search space for combined attacks is strictly larger under mild assumptions. 

\begin{theorem}[Informal] For an image $\mathcal{I}$, consider the two threat models defined by an $\linf$ perturbation of at most $\delta$ and a flow attack where grid sampling is allowed within an $\linf$ ball of size $\epsilon$. Then if there exist both low and high-contrast areas in $\mathcal{I}$, then there exists some perturbed image $\mathcal{I}^\prime$ that is attainable via a combination of the aforementioned threat models but is not attainable by each threat model alone. 
\end{theorem}
We defer the formal theorem statement and proof to the appendices.

\section{Experiments} 
\begin{table*}
  \centering
    \caption{Description of attacks, corresponding threat model and optimization procedure used.}\label{tab:not}
  \begin{tabular}{cccc}
    \toprule
    \textbf{Attack}     & \textbf{Type}     & \textbf{Threat Model} & \textbf{Optimization Procedure}\\
    \midrule
    Delta & Additive transformation  & $\linf \leq 8$ & FGSM/PGD    \\
    R     & Rotation about the center & Angle $|\theta| \leq \pi/24$  & PGD    \\
    T     & Translation (horizontal and vertical)       & Distance $\leq$ 3.2 pixels &PGD  \\
    StAdv     & Flow networks       & Distance $\leq$ 1.6 pixels  &PGD with TV loss ($\tau=0.05$) \\
    \bottomrule
  \end{tabular}
\end{table*}

\subsection{Adversarial Example Toolbox} 

To support our experiments and make it easier for others to experiment with adversarial examples, we have developed an extensive adversarial example toolbox that incorporates a variety of state-of-the-art attacks, training techniques and evaluation methods. This toolbox, known as \texttt{mister\_ed} is built on top of the popular PyTorch machine learning framework \cite{paszke2017automatic} and is intended to be a PyTorch equivalent to the \texttt{cleverhans} library \cite{papernot2018cleverhans}. This toolbox fundamentally incorporates optimization over the \emph{parameter space} of transformations, as opposed to relying on simply additive transformations. This toolbox and tutorials are hosted on a public \href{https://github.com/revbucket/mister_ed}{github repo}.

\subsection{Experimental Setup} 
We experiment on classifiers trained on CIFAR-10 and ImageNet. For experiments on CIFAR-10, we use a ResNet32 architecture using pretrained weights that has 92.63\% accuracy on the validation set \cite{krizhevsky2009learning, he2016deep, cifar2018pretrained}. For experiments on ImageNet, we use a NASNet-A Mobile architecture using pretrained weights that has 74.08\% accuracy on the validation set \cite{nasnetmobile, ILSVRC15, zoph2017learning}. Adversarial examples generated using our methods on CIFAR-10 are displayed in Figure \ref{cifar_attack_suite} and further examples on both datasets are displayed in the appendices. We primarily discuss results for CIFAR-10 since we have found that ImageNet classifiers are extremely vulnerable to even the most basic adversarial attacks.
\begin{figure*}
\centering
        \hfill
    \subfigure{\includegraphics[width=0.9\columnwidth]{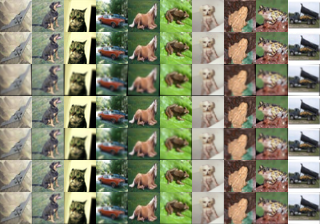}}
    \hfill 
    \subfigure{\includegraphics[width=0.9\columnwidth]{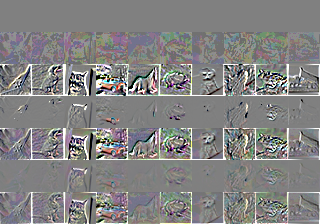}}
    \hfill
    \caption{Adversarial examples on a CIFAR-10 trained ResNet32 defended against delta addition and StAdv attacks. The left figure are the adversarial examples, with the rows being various styles of attack: [original, delta, rotation+translation, stAdv, delta+rotation+translation, delta + stAdv, delta + StAdv + rotation+translation]. The right figure is the difference between the attacked image and the original image magnified by 5.}
    \label{cifar_attack_suite}
\end{figure*}

Unless otherwise noted, we use Projected Gradient Descent (PGD) to generate our attacks. We use the Carlini-Wagner style loss using $f_6$ with $\kappa=0$ \cite{Carlini2017-qm} as the adversarial loss, presenting results for whichever is stronger. For all flow-network based attacks, we use the total variation flow loss to enforce local smoothness, with $\tau=0.05$ and refer to this style of attacks as StAdv \cite{xiao2018spatially}.  We perform PGD until the loss converges, with a minimum of 100 iterations, using Adam as an optimizer with learning rate of $0.01$, projecting into the hard constraints defined by the threat model at each iteration. Our FGSM attack has a step size of 8 and projects the attacked example onto the range of valid images after stepping. 

We have found that the ordering of the layers does not have a significant effect on the attack's effectiveness, but when describing combined attacks as Layer1 + Layer2, we perform the attack in Layer1 first and then apply Layer2. We optimize all sets of parameters for multilayer threat models simultaneously.

\begin{table*}[ht]
\centering
\caption{The table shows the accuracy of defended networks (rows) for various attacks (columns). Note that lower values imply stronger attacks and weaker defenses. For example, the entry with row FGSM and column StAdv indicates that 0.39\% of the samples were successfully defended by the network adversarially trained using FGSM, implying that StAdv is a strong attack against FGSM and FGSM is a weak defense for StAdv. As can be seen, the proposed attack (last column) is significantly stronger than all previously known attacks. 
}
\label{table:attack_vs_defenses}
  \begin{tabular}{c|lllllll}
    \toprule
                       & \textbf{Ground} & \textbf{R + T} & \textbf{FGSM} & \textbf{Delta} & \textbf{\begin{tabular}[c]{@{}l@{}}Delta + \\ R + T\end{tabular}} & \textbf{StAdv} & \textbf{\begin{tabular}[c]{@{}l@{}}Delta + \\ StAdv\end{tabular}} \\ 
                           \midrule

\textbf{Undefended}    & 92.4            & 56.02          & 18.02         & \textbf{0.00}  & \textbf{0.00}                                                     & 1.75           & \textbf{0.00}                                                     \\ 
\textbf{FGSM}          & 90.07           & 62.11          & 95.01         & 0.49           & 0.69                                                              & 2.56           & \textbf{0.05}                                                     \\ 
\textbf{Delta}         & 81.28           & 58.91          & 47.83         & 38.14          & 18.70                                                             & 14.02          & \textbf{5.42}                                                     \\ 
\textbf{Delta + R + T} & 83.53           & 70.27          & 47.45         & 34.29          & 28.24                                                             & 16.05          & \textbf{4.97}                                                     \\ 
\textbf{StAdv}         & 83.62           & 71.41          & 18.33         & 0.36           & 0.39                                                              & 23.90          & \textbf{0.07}                                                     \\ 
\textbf{Delta + StAdv} & 82.07           & 63.98          & 47.06         & 36.71          & 24.82                                                             & 24.30          & \textbf{9.74}                                                     \\ \bottomrule
\end{tabular}
\end{table*}

\subsection{Attack Effectiveness Against Adversarial Training}

Here we consider various CIFAR-10 classifiers trained via an adversarial training procedure \cite{Goodfellow2014-vh} as a means of defense against adversarial attacks. We then evaluate these defenses against both existing attacks as well as several sequential compositions of attacks. We use combinations of layers with the shorthand notations and hard constraints mentioned in Table \ref{tab:not}. The results are shown in Table \ref{table:attack_vs_defenses}.

\paragraph{Different attack layers are ``orthogonal.''} 
Combining attack layers yields remarkably powerful attacks as we have greatly expanded the set of valid adversarial images by relaxing the threat model. In particular, spatial attacks are strong against delta defended networks and vice versa. For example, the StAdv trained network can only successfully defend against delta attacks in 0.36\% of validation images. While a delta trained network is also weak against an StAdv attack (14.02\%), it is the strongest defense against orthogonal attacks. Further, the strengths of disjoint attacks seem to stack: rotation and translation alone perform poorly on a delta defended network (58.91\%), though when applied in combination with a delta attack, they decrease classifier accuracy from 38.14\% to 18.70\%. Since both delta and StAdv attacks are quite powerful by themselves, combining them should yield an even more powerful attack that performs well against both delta defended networks and StAdv defended networks. Our experimental results confirm this. The Delta + StAdv combined attack is the strongest of all attacks evaluated, even against a network trained against attacks of this style alone (9.74\%).

\paragraph{A larger threat model does not necessarily imply a stronger defense.}
Notice that every attack is best defended by a network trained against that class of attacks only. In other words, for any particular single attack style, the network that best defends against this attack is the network that saw adversarial examples of that style only. For example, a Delta + R + T defended network performs worse against a Delta attack (34.29 \%) than just a Delta defended network (38.14 \%). 

\subsection{Combined Attacks Are Strong Without the Perceptual Cost}

Observe the following simple fact: as we relax the threat model, say by increasing the allowable $\ell_\infty$ perturbations in a delta addition attack, the attack strength cannot decrease, since any attack allowed in the stricter threat model must also be allowed under the more relaxed threat model. Allowing a larger $\ell_\infty$ perturbation, however, could increase the perceptual distance of an adversarial example, suggesting that there is a tradeoff between the looseness of the threat model and perceptual distance of generated adversarial examples.

Our goal is to find a threat model that generates strong adversarial examples that also have minimal perceptual distortion. To do so, we quantify attack strength as the percentage of successful attacks against correctly classified images from a given data set. We quantify perception with both LPIPS and (1-SSIM) \cite{Zhang2018-mw, Wang2004-ky}. Note that for ease of exposition, we multiply LPIPS and (1-SSIM) distance by 1000 and 100 respectively. We try various combinations of threat models of stAdv and delta addition attacks on the network that has been defended against the combination of stAdv and delta attacks and summarize our results in Table \ref{percept_combo_table}. We display a heatmap describing the strength and LPIPS distance as a function of these bounds in Figure \ref{heatmaps}. Note that we use our best defended network for these experiments as undefended networks are so susceptible to adversarial attacks that they do not generate interpretable results. 
\begin{figure*}[ht!]
\centering
        \hfill
    \subfigure{\includegraphics[width=0.65\columnwidth]{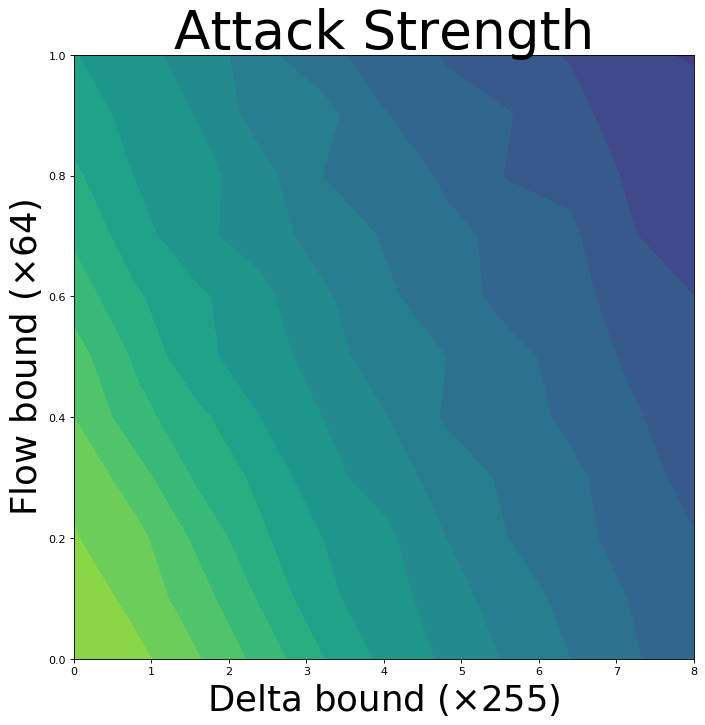}}
    \hfill 
    \subfigure{\includegraphics[width=0.65\columnwidth]{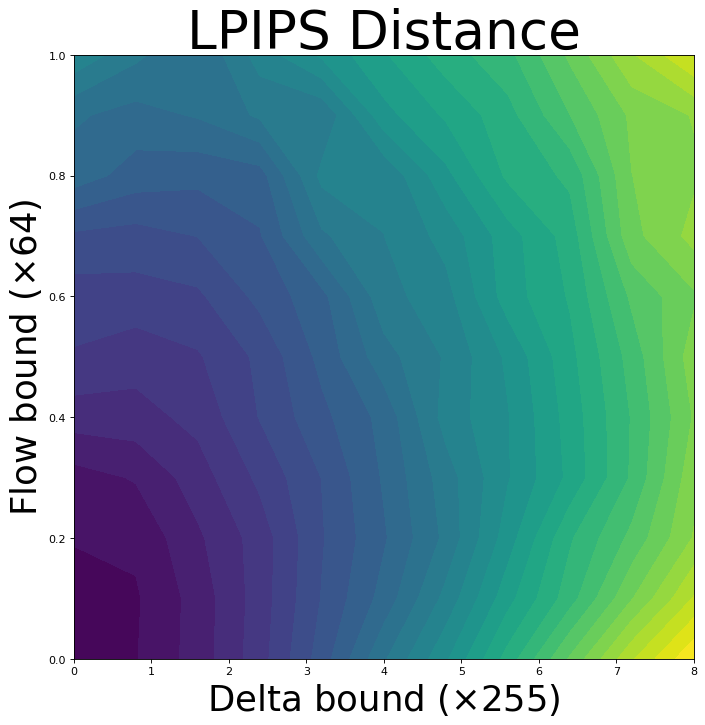}}
    \hfill
    \caption{Heatmaps of accuracy of a CIFAR-10 classified ResNet defended against combinations of additive and flow attacks (left) and the average LPIPS distance of the generated adversarial images (right). The $x$-axis of each figure corresponds to allowable additive noise, while the $y$-axis corresponds to allowable flow. As one travels away from the origin, attacks get stronger (accuracy lowers) and the perceptual distance increases. Note that moving from dark to light colour corresponds to increase in value. Crucially, along an isoline on the left, the minimal LPIPS distance is not obtained on either axis, implying combination attacks yield more perceptually indistinguishable results.}
    \label{heatmaps}
\end{figure*}

\begin{table}
\centering
\caption{Accuracies and perceptual distance of various threat models of combined delta+stAdv attacks against a network defended against delta+stAdv attacks. The Delta column refers to how much additive noise is allowed, in terms of $\ell_\infty$, and the Flow column refers to the $\ell_\infty$ bound on how much flow is allowed, in terms of pixels. The accuracy column measures the accuracy of the defended network against the attack, so smaller means a stronger attack. Our central finding is illustrated in the last row: A combination attack is both stronger and has better similarity as measured by both LPIPS and SSM.  
} 

\label{percept_combo_table}
\begin{tabular}{ccccc}
    \toprule
\textbf{Delta} & \textbf{Flow} & \textbf{Accuracy} & \textbf{LPIPS} & \textbf{(1 - SSIM)} \\ 
    \midrule
0              & 1.6             & 24.30             & 5.00         & 3.63           \\ 
8              & 0             & 36.09             & 2.83         & 1.79           \\ 
8              & 1.6             & 9.74              & 4.13         & 2.75           \\ 
4            & 0.8           & 28.20             & 2.67         & 1.67           \\ 
2           & 1.6           & 22.03             & 3.45         & 2.57           \\ 
8            & 0.4          & 22.81             & \textbf{2.63}         & \textbf{1.56}           \\ 
    \bottomrule
\end{tabular}
\end{table}

The strongest threat model is the one that allows the largest perturbations, which is able to generate examples that are classified correctly 9.74\% of the time. However, observe that this threat model generates examples that are not as perceptually indistinguishable as more constrained threat models. It is interesting to note, however, that augmenting flow attacks with delta attacks results in more perceptually indistinguishable attacks. On the other hand, incorporating a small amount of allowable flow to an additive-only attack both decreases classifier accuracy and perceptual distortion, making it a strictly better attack in both senses. This phenomenon is also present when incorporating a small amount of additive noise in an otherwise pure flow attack, though it is not as pronounced. Furthermore, notice that decreasing LPIPS values correspond with decreasing (1 - SSIM) values, suggesting that these results are likely independent of the perceptual metrics.

The heatmaps present in Figure \ref{heatmaps} display the strength and LPIPS distance of combined flow and additive attacks as a function of the bounds of the threat model of each. Observe the near-linearity in contours of the strength heatmap, indicating that there is a natural tradeoff between flow and additive noise yielding attacks of equivalent strength. The LPIPS heatmap, however tells a different story - notably, there are segments where the contours are near-perpendicular to the strength contours. This lends further credence to the idea that it is possible to generate a stronger and more perceptually indistinguishable attack by utilizing a combination threat model than either threat model in isolation. 

To further support the claim that combination attacks can generate more perceptually indistinguishable attacks, we perform a modified Carlini-Wagner attack where the distance metric used is the LPIPS perceptual metric, and two (unbounded) threat models are considered: one using only a parameterization of additive noise, and another using a parameterization of a spatial and additive attack. Such an attack is certainly overpowered due to the unboundedness in the threat model, achieving misclassification in all input images considered. The corresponding perceptual distance of each of these attacks against our Delta defended network is displayed in table \ref{lpips_cw}. Observe that directly regularizing against perceptual metrics generates vastly more imperceivable attacks according to both metrics considered. Interestingly, the combination attacks yield perceptual distances that are significantly smaller than their corresponding attacks parameterized only by an additive delta, which holds over both perceptual metrics.

\begin{table}
\centering
\caption{Carlini Wagner style attacks performed with LPIPS used as a distance metric against our Delta defended network. 
Both of these attacks are successful on 100\% of the examples considered; however when the parameterization of transformations includes a spatial transformer, the adversary is able to generate examples with lower perceptual distance.
} 

\label{lpips_cw}
\begin{tabular}{ccc}
    \toprule
\textbf{Parameterization} & \textbf{LPIPS distance} & \textbf{(1-SSIM)} \\ 
    \midrule
Delta              & 0.231             & 0.0383            \\ 
Delta + Flow       & 0.158            & 0.0245               \\ 

    \bottomrule
\end{tabular}
\end{table} 

\section{Conclusions} 
In this paper, we aim to generate adversarial examples of minimal distortion under a perceptual metric. To our knowledge, we are the first to employ these metrics, commonly used in several other applications, to the adversarial example domain. We present a framework that allows for the combination of various styles of adversarial attacks in order to better parameterize the space of perceptually similar images. We support this by first proving that combined attacks can result in images previously unattainable by individual attacks. Empirically we apply our combination attacks to show that adversarial defenses do not transfer their robustness properties to unseen attack styles. Finally, we demonstrate that combination attacks can generate adversarial examples that are both significantly stronger and more perceptually indistinguishable than any pure attack alone. 

\bibliography{perceptual_aes}
\bibliographystyle{icml2019}

\begin{appendices}

\section{Threat Model Search Space}
We first describe notation we use throughout the proofs and tie it back to the perturbations we experiment with. We then formulate and prove our main result.

\subsection{Notation and Preliminaries}

In this section, we treat an image as a two-dimensional grid of pixels. Each pixel has some value corresponding to its color value, which is some number between $0$ and $255$. We define a quadrant associated with pixel $x$ to be a grid of four pixels defined by pixels in the neighborhood of $x$. Figure \ref{fig:quadrants} depicts four quadrants surrounding a reference pixel $x_{00}$. If our reference pixel for some fixed quadrant is $x$, then we relabel this pixel $x_{00}$. The remaining pixels in the quadrant are then relabeled according to Figure \ref{fig:bi_interpolate}.

\begin{figure}[H]
    \centering
    \includegraphics[scale=0.25]{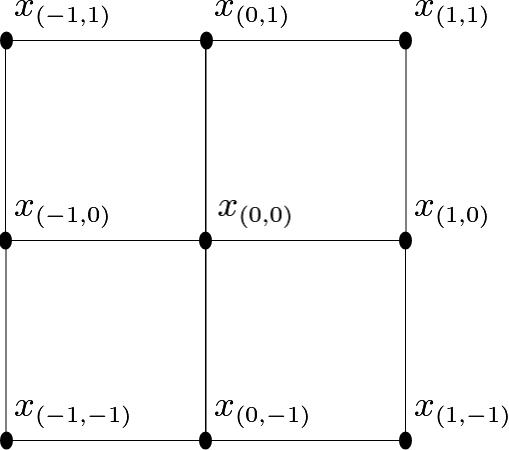}
    \caption{The figure above depicts four quadrants surrounding some interior reference pixel $x_{(0, 0)}$. Note that border or corner pixels will not have four surrounding quadrants.}
    \label{fig:quadrants}
\end{figure}




Our result is dependent on the changes between pixels in some neighborhood. The following definition formalizes the notion of local change we use throughout the remainder of this section.

\begin{definition}[Local changes]
We define the following quantities pertaining to local changes with respect to some reference pixel. We first define the local change between the reference pixel and its neighboring pixels:
\[
C_{\max}(x_{00}) \coloneqq \max_{i, j \in \setmath{-1, 0, 1}} \abs{x_{ij} - x_{00}} 
\]
We also consider the maximum local change while only taking into account the off-diagonal pixels. This can be written below:
\[
E_{\max}(x_{00}) \coloneqq \max_{i,j \in \{-1, 0, 1\} \text{s.t.} |i| \neq |j|}\abs{x_{ij} - x_{00}}
\]
\end{definition}

We also treat perturbations as functions $f:X\times \Theta \rightarrow X$ that take as input a pixel $x$ in the space of possible pixels $X$ and perturbation parameters $\theta$ and outputs another pixel. For clarity, we denote the flow perturbation function as $f_{\text{flow}}$ and the $\linf$ perturbation as $f_{\text{add}}$. Furthermore, observe that a perturbation's parameters $\theta$ must be constrained by some preset threat model. 

We now describe the perturbations we consider in our result.

\paragraph{Quantifying Flow}

Suppose our threat model allows a maximum flow perturbation at each pixel to be $\eps$. Thus, for any pixel, observe the following bounds on the horizontal and vertical flows (respectively) for a given pixel:
$$0 \le \eps_h, \eps_v \le \eps$$
Now, recall that a flow perturbation entails performing grid sampling from the grid of images, where the sampling is performed by bilinear interpolation. In other words, a pixel can flow into one of the (at most) four quadrants depicted in Figure \ref{fig:quadrants}. To explicitly quantify the value our new pixel assumes upon performing the bilinear interpolation, we state and prove the following key known equality we use throughout our proofs. 

\begin{figure}[H]
    \centering
    \includegraphics[scale=0.40]{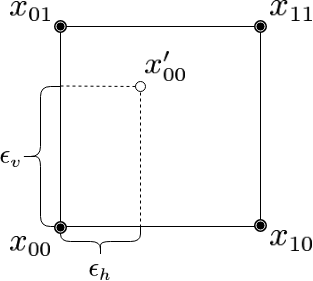}
    \caption{The figure above depicts a single quadrant with reference pixel $x_{00}$. If we wish to perturb $x_{00}$ to $x_{00}^\prime$ using a flow perturbation, then $\eps_h$ and $\eps_v$ give the necessary horizontal flow and vertical parameters, respectively. This figure also depicts our pixel labeling convention once we select a quadrant.}
    \label{fig:bi_interpolate}
\end{figure}

\begin{lemma} 
\label{lemma:flow_function}
If a pixel $x_{00}$ is perturbed by a flow attack into a quadrant whose corners are $x_{00}, x_{01}, x_{10}$, and $x_{11}$, and if the horizontal and vertical perturbations are given by $\eps_h$ and $\eps_v$ respectively, then 
\begin{align*}
f_{\text{flow}}(x_{00}, (\eps_h, \eps_v)) &= x_{00} + (x_{10} - x_{00})\paren{\paren{1 - \eps_v}\eps_h}\\
&\quad + (x_{01} - x_{00})\paren{\eps_v\paren{1 - \eps_h}}\\
&\quad+ (x_{11} - x_{00})\paren{\eps_v\eps_h}.
\end{align*}
\end{lemma}
\begin{proof}
We use the bilinear interpolation procedure. Without loss of generality, we apply the horizontal perturbations first. The endpoints of the line we must vertically interpolate over are now:
$$x_{00}\paren{1 - \eps_h} + x_{10}\eps_h \quad \text{and}\quad x_{01}\paren{1 - \eps_h} + x_{11}\eps_h.$$
Performing the vertical interpolation yields:
\begin{align*}
    &\ \paren{1 - \eps_v}\paren{x_{00}\paren{1 - \eps_h} + x_{10}\eps_h} + \eps_v\paren{x_{01}\paren{1 - \eps_h} + x_{11}\eps_h} 
\end{align*}
Rearranging the above gives the result.

\end{proof}
Note that the magnitude of this perturbation can be expressed in terms of the pixel differences between the pixels in the appropriate quadrant and the reference pixel:
\begin{align*}
\abs{f_{\text{flow}}(x_{00}, (\eps_h, \eps_v)) - x_{00}} &= |(x_{10} - x_{00})\paren{\paren{1 - \eps_v}\eps_h}\\
&\quad+ (x_{01} - x_{00})\paren{\eps_v\paren{1 - \eps_h}}\\
&\quad + (x_{11} - x_{00})\paren{\eps_v\eps_h}|.
\end{align*}

Observe that as a consequence of the way we characterize flow perturbations, our flow upper bound $\eps$ must lie in $[0, 1]$. Thus, the set of all allowable outputs using a flow threat model with parameter $\eps$ is:
$$\setmath{f_{\text{flow}}(x, (\eps_h, \eps_v)) \suchthat 0 \le \eps_h, \eps_v \le \eps}$$
We can also specify the perturbation for the $\delta$ addition transformation:
$$f_{\text{add}}(x, \delta) = x + \delta$$
As a result, the set of all allowable outputs using a $\delta$ addition threat model with parameter $\delta$ is:
$$\setmath{f_{\text{add}}(x, \delta_a) \suchthat \abs{\delta_a} \le \delta}$$

\subsection{Main Theorem and Proof}

We now state and prove our main result.

\begin{theorem}[Formal version of Theorem 1] 
Consider an image $\mathcal{I}$ and associated $\linf$ and flow perturbation constraints of $\delta$ and $\eps$, respectively. Suppose there exist two distinct pixels $p, q \in \mathcal{I}$ such that $C_{\max}(p) < \frac{\delta}{2\eps}$ and $E_{\max}(q) \ge \frac{\delta}{\eps}$. Then, we can apply an $\delta$ addition to $p$ and a flow perturbation to $q$ which results in a new image $\mathcal{I}^\prime$ that is not attainable by using solely a $\delta$ addition or solely a flow perturbation.
\end{theorem}
\begin{proof}
To prove our main theorem, it suffices prove the following two lemmas. 
\begin{figure*}[t]
    \centering
    \subfigure{\includegraphics[width=0.47\columnwidth]{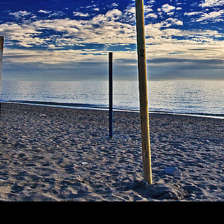}}
    \hfill
    \subfigure{\includegraphics[width=0.47\columnwidth]{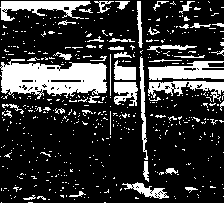}}
    \hfill
    \subfigure{\includegraphics[width=0.47\columnwidth]{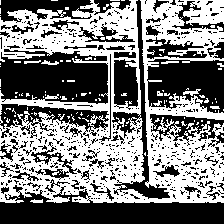}}
    \hfill
    \subfigure{\includegraphics[width=0.47\columnwidth]{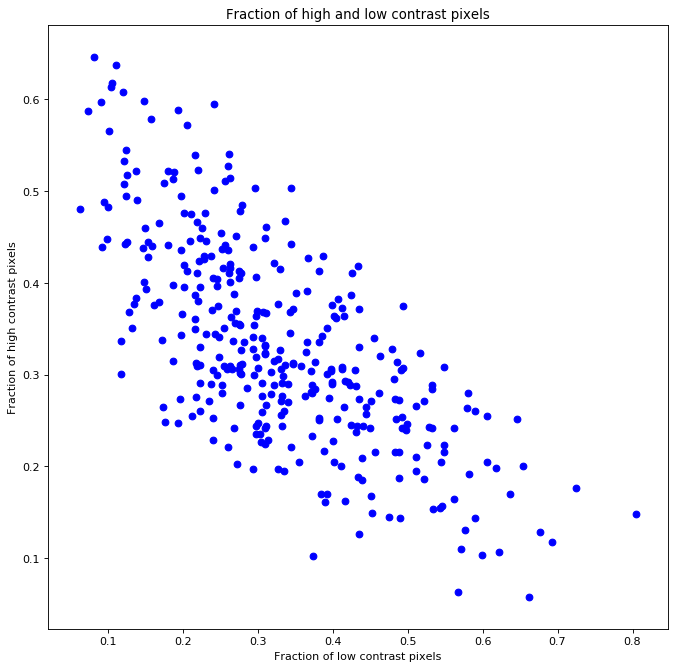}}
    \caption{Left three images: An example from ImageNet highlighting the low and high contrast conditions we impose. The leftmost image is the original image. We then take the blue channel and apply our condition to the resulting pixel values. The middle image highlights in white the areas of low contrast. The right image highlights in white the areas of high contrast. Right: Scatter plot depicting the percentage of pixels satisfying each condition across 384 randomly sampled examples from CIFAR-10.}
    \label{fig:hilo_contrast}
\end{figure*}
The first lemma characterizes regions of our image that may benefit more from additive changes than from geometric changes. Such regions of the image may exhibit local smoothness, and as a result, there is not much geometry present to modify. 
\begin{lemma}
\label{lemma:flow_sucks}
Suppose that for some pixel $x_{00}$, we have $C_{\max}(x_{00}) < \frac{\delta}{2 \eps}$. Then, it follows that an attack using a $\delta$ addition results in a greater perturbation than an attack using at most $\eps$ flow. 
\end{lemma}
\begin{proof}
To prove the lemma, we must show that for all flows, the maximum pixel change is less than $\delta$. Without loss of generality, assume the flow moves to the top right quadrant. Then we have:
\begin{align*}
&\abs{f(x_{00}, (\eps_h, \eps_v)) - x_{00}}\\
&= \left|\paren{x_{10} - x_{00}}\paren{\eps_h\paren{1 - \eps_v}} + \paren{x_{01} - x_{00}}\paren{\eps_v\paren{1 - \eps_h}}\right.\\
&\quad+ \left.\paren{x_{11} - x_{00}}\eps_h\eps_v\right|\\
&\leq \abs{x_{10} - x_{00}}\paren{\eps_h\paren{1 - \eps_v}} + \abs{x_{01} - x_{00}}\paren{\eps_v\paren{1 - \eps_h}}\\
&\quad+ \abs{x_{11} - x_{00}}\eps_h\eps_v\\
& \leq C_{\max}(x_{00})\paren{\paren{1 - \eps_h}\eps_v + \paren{1 - \eps_v}\eps_h + \eps_v\eps_h}\\
& = C_{\max}(x_{00})\paren{\eps_h + \eps_v - \eps_v\eps_h}\\
& \le  2 \eps C_{\max}(x_{00}) < \delta
\end{align*}
Here the second inequality follows from observing that for any possible pixel change in the neighborhood, $\abs{x_{ij} - x_{00}} \leq C_{\max}(x_{00}) < \delta$. The third inequality follows from setting $\eps_h, \eps_v \leq \eps$.
\end{proof}

The following lemma characterizes regions of our image that are not locally smooth. Such regions of the image have extensive geometry that a flow attack might be able to exploit in ways that an $\linf$ attack cannot. We explicitly describe one such sufficient condition for this to occur.
\begin{lemma}
\label{lemma:linf_sucks}
Consider some reference pixel $x_{00}$, if $E_{\max}(x_{00}) > \frac{\delta}{\eps}$, a flow attack on $x_{00}$ with limit $\delta$ can result in a greater pixel change than an additive change of at most $\delta$ on $x_{00}$.
\end{lemma}
\begin{proof}
Without loss of generality, suppose $\abs{x_{01} - x_{00}} \ge \frac{\delta}{\eps}$. Set $\eps_h = 0$ and $\eps_v = \eps$. Using Lemma \ref{lemma:flow_function}, we have
\begin{align*}
    &\abs{f(x_{00}, (0, \eps)) - x_{00}}\\
    & =|\paren{x_{10} - x_{00}}\paren{\eps_h\paren{1 - \eps_v}} + \paren{x_{01} - x_{00}}\paren{\eps_v\paren{1 - \eps_h}}\\
    &\quad+ \paren{x_{11} - x_{00}}\eps_h\eps_v| \\
    &= \eps\abs{\paren{x_{01} - x_{00}}} > \eps \cdot \frac{\delta}{\eps} > \delta.
\end{align*}
This gives us a flow attack that is more powerful that a $\delta$ additive attack.
\end{proof}
Consider the pixels in our image satisfying the conditions required by Lemma \ref{lemma:flow_sucks}. By using a $\delta$ addition on these pixels instead of using a flow perturbation, we can achieve a greater pixel difference between the perturbed image and the original image than any flow attack. Similarly, we can achieve greater pixel differences if we use a flow perturbation on the pixels satisfying the conditions required by Lemma \ref{lemma:linf_sucks} instead of a $\delta$ addition. Since by the assumption in the theorem, both these sets are non-empty, we can produce an image that was not attainable solely by either perturbation.
\end{proof}

\subsection{Existence of High and Low Contrast Areas}
It is important to note that for our theorem statement to hold, we require that there exist low contrast areas as well as high contrast areas. To verify this, we inspected 384 randomly selected examples from CIFAR-10 and counted how many pixels satisfy our low or high contrast conditions. Our results are detailed in Figure \ref{fig:hilo_contrast}. Notably, a large fraction of pixels in each image we sampled satisfies one of our constraints, and every image contains at least one high and low contrast area\footnote{The two constraints are disjoint, so at most one can be satisfied for a given pixel.}. This indicates that combining these attacks gives us the power of exploring a larger set of potential adversarial images compared to any one attack individually.

\section{Images}

\label{appendix:images}
\begin{figure*}[ht]
    \centering
    \subfigure{\includegraphics[width=12cm]{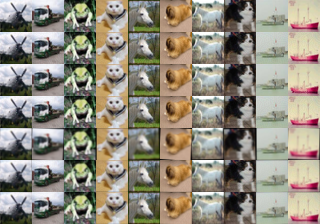}}
    \vfill 
    \subfigure{\includegraphics[width=12cm]{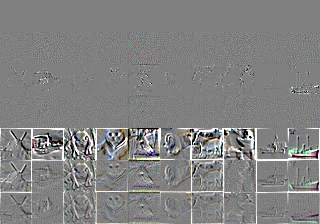}}
    \caption{Top: Various CIFAR-10 images with various successful attacks against an undefended network on each row. From top to bottom these are: originals, delta, StAdv, delta + StAdv, rot+trans, delta + rot + trans, delta + StAdv + rot + trans. Bottom: The differences between originals magnified by 5}.\label{fig:cifar_suite_undefended}
\end{figure*}    

\begin{figure*}[ht]
    \centering
    \subfigure{\includegraphics[width=12cm]{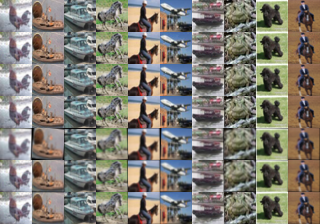}}
    \vfill 
    \subfigure{\includegraphics[width=12cm]{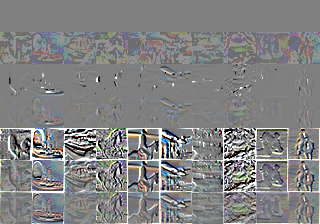}}
    \caption{Same figure as \ref{fig:cifar_suite_undefended} but with attacks performed against a delta-trained network}.
\end{figure*}    

\begin{figure*}[ht]
    \centering
    \subfigure{\includegraphics[width=12cm]{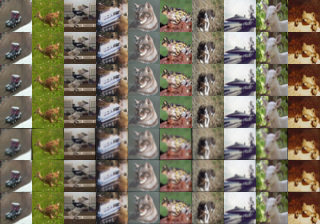}}
    \vfill 
    \subfigure{\includegraphics[width=12cm]{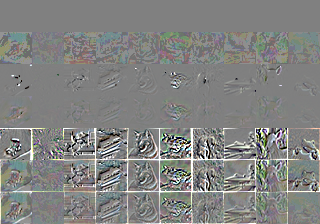}}
    \caption{Same figure as \ref{fig:cifar_suite_undefended} but with attacks performed against a delta + stadv trained network}.
\end{figure*}   

\newpage
\begin{figure*}[ht]
    \centering
    \subfigure{\includegraphics[scale=0.25]{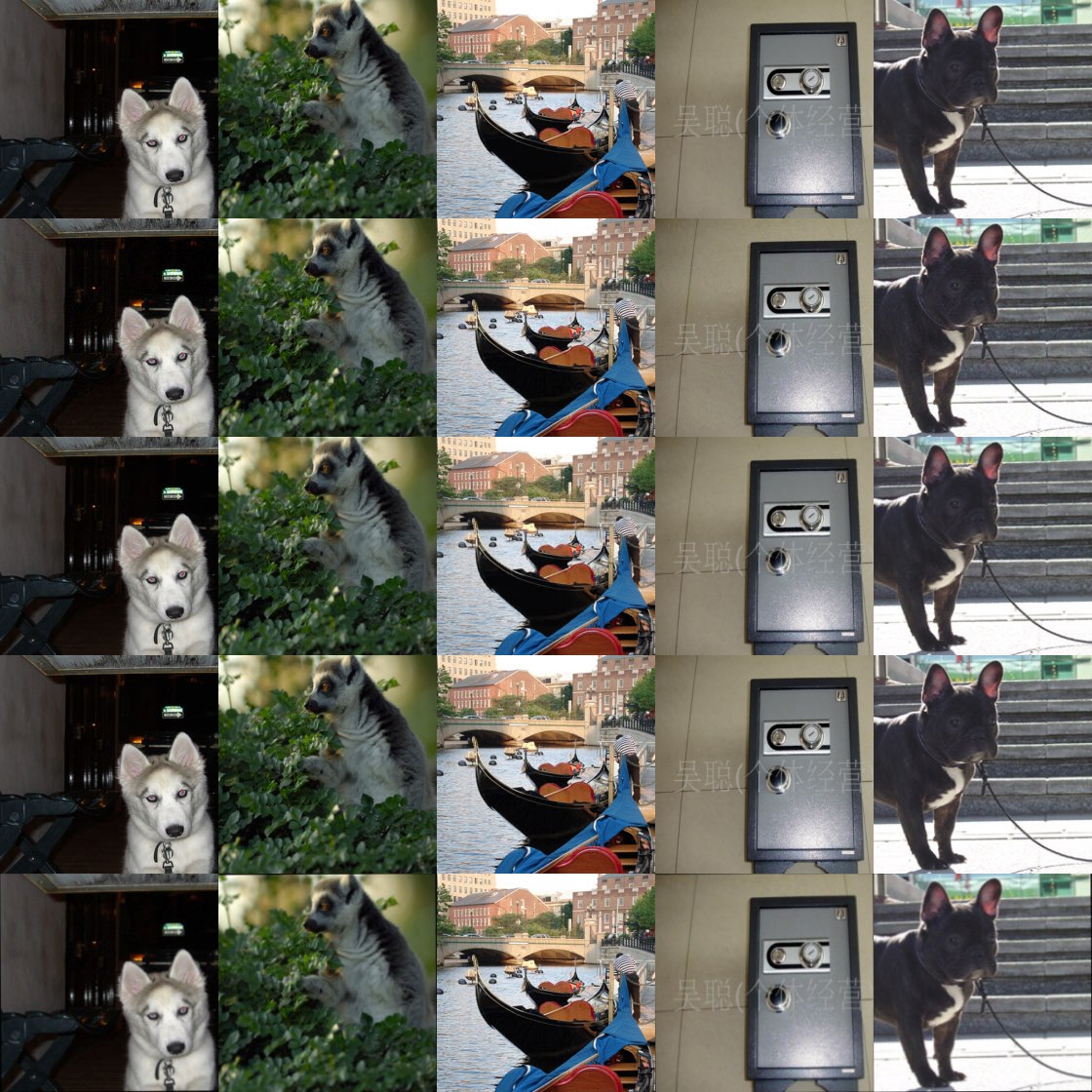}}
    \vfill 
    \subfigure{\includegraphics[scale=0.25]{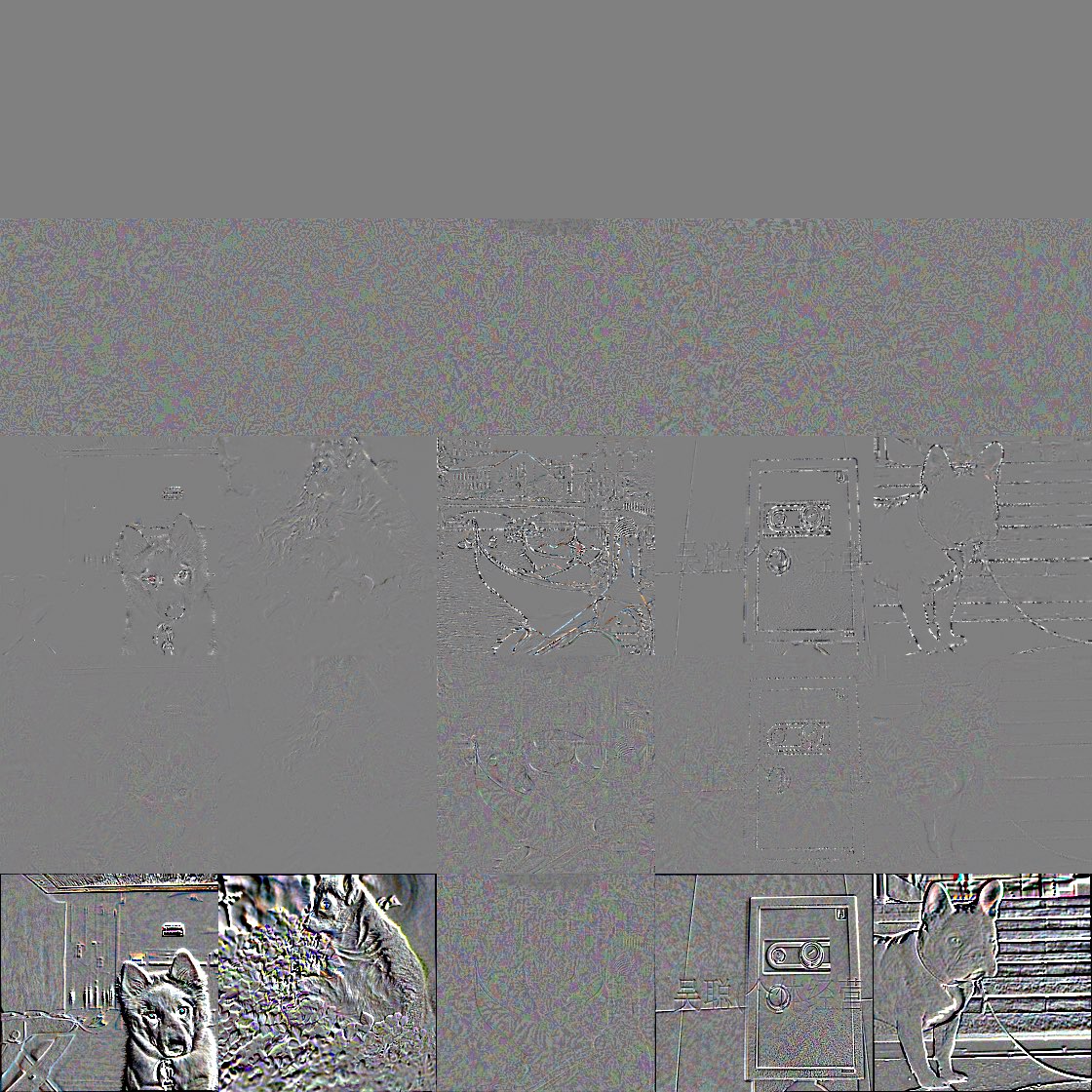}}
    \caption{Top: Various ImageNet images with various successful attacks against an NasnetAMobile network on each row. From top to bottom these are: originals, delta, stAdv, delta + stAdv, rot+trans, delta + rot + trans. Bottom: The differences between originals magnified by 10. From this it is clear that completely imperceivable perturbations are all that is necessary to fool undefended ImageNet networks}\label{fig:imagenet_suite}
\end{figure*}    

\begin{figure*}[ht]
    \centering
    \subfigure{\includegraphics[scale=0.25]{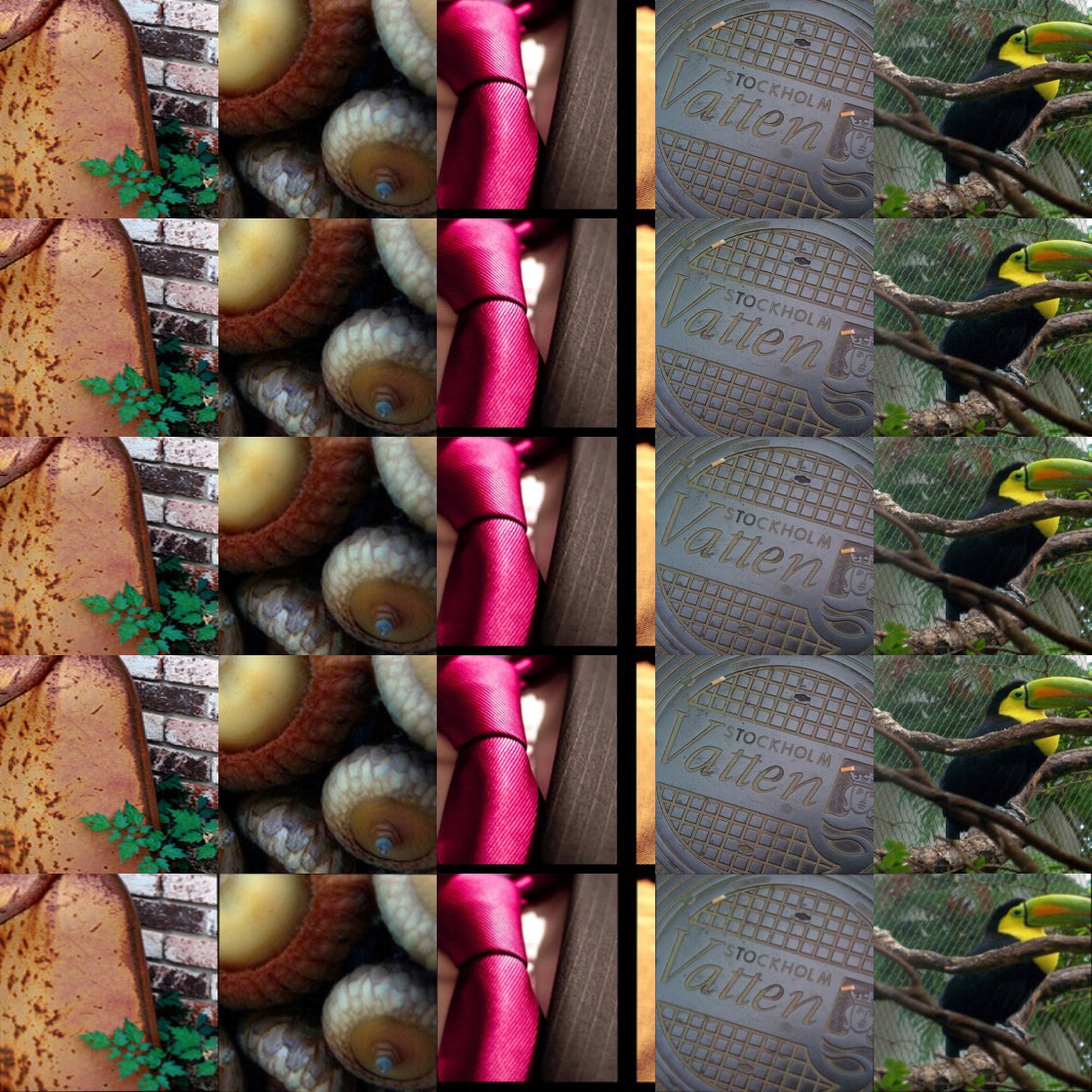}}
    \vfill 
    \subfigure{\includegraphics[scale=0.25]{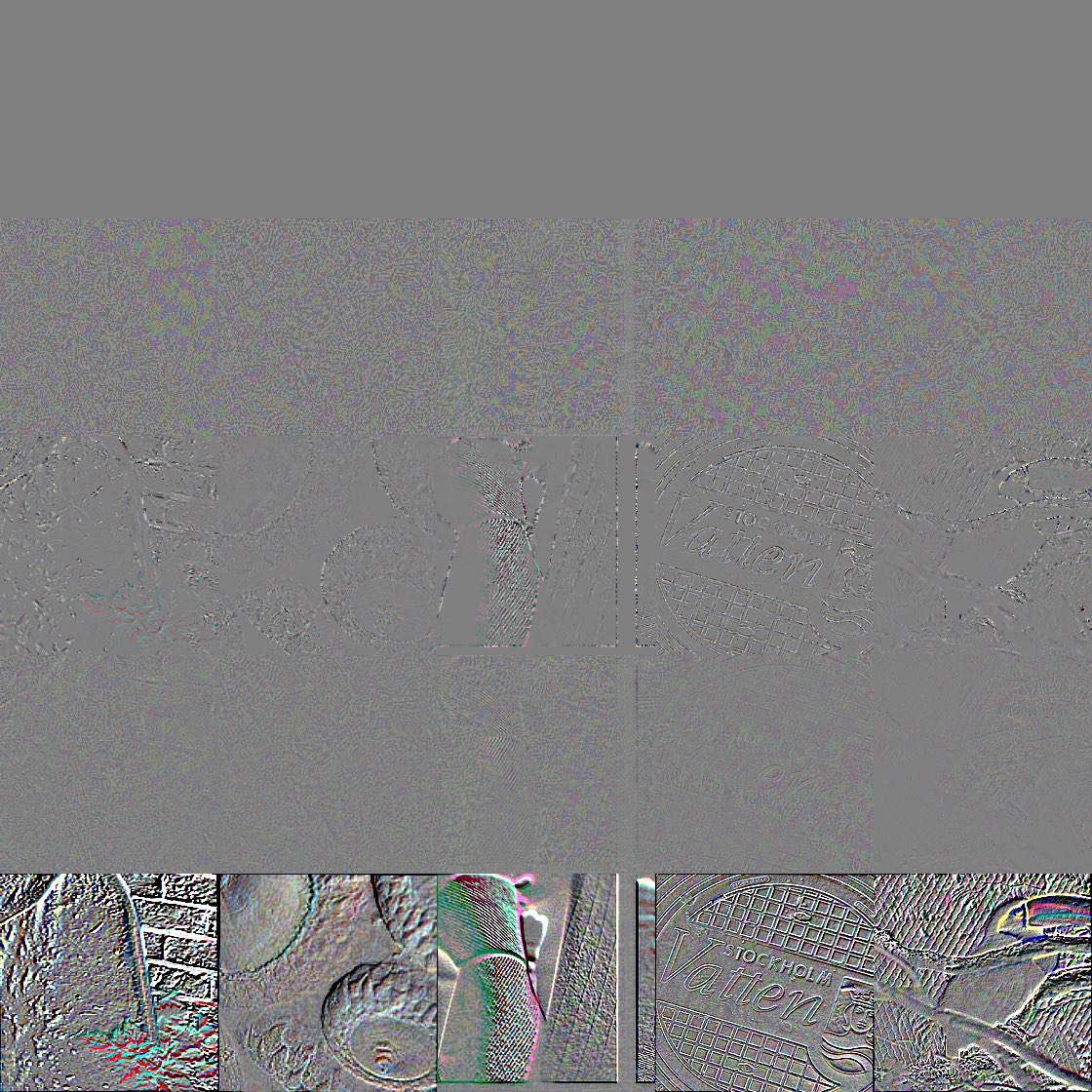}}
    \caption{Same as figure \ref{fig:imagenet_suite}}
\end{figure*}
\end{appendices}



\end{document}